\documentclass{article}

\usepackage[preprint]{neurips_2026}

\usepackage[utf8]{inputenc} 
\usepackage[T1]{fontenc}    
\usepackage{hyperref}       
\usepackage{url}            
\usepackage{booktabs}       
\usepackage{amsfonts}       
\usepackage{nicefrac}       
\usepackage{microtype}      
\usepackage{xcolor}         
\usepackage{amsmath}

\usepackage{pgfplots}
\usepackage{pgfplotstable}

\newcommand{\qset}{\mathcal{Q}}

\newcommand{\aset}{\mathcal{A}}
\newcommand{\sset}{\mathcal{S}}

\newcommand{\polset}{\mathcal{P}}

\newcommand{\region}{\graval}

\newcommand{\pol}{\pi}

\newcommand{\qval}{q}

\newcommand{\qopt}{\qval_{\best{\pol}}}

\newcommand{\graval}{Q}

\DeclareMathOperator*{\E}{\mathbb{E}}

\newcommand{\worst}[1]{\overline{#1}} 
\newcommand{\best}[1]{{#1}_*}

\newcommand{\kpz}{{k}}

\usepackage{amsthm}

\usepackage{amsmath}
\usepackage{amssymb}
\usepackage{amsfonts}
\newtheorem{theorem}{Theorem}

\newtheorem{lemma}{Lemma}

\usepackage{pdfpages}

\usepackage{enumitem}

\usepackage{mathtools}
\usepackage{bm}

\usepackage{multirow}

\definecolor{cambridgeblue}{RGB}{163, 193, 173}
\definecolor{CambridgeSuperLightBlue}{RGB}{204, 227, 245}
\definecolor{CambridgeSuperLightOrange}{RGB}{252, 229, 205}
\definecolor{CambridgeSuperLightPurple}{RGB}{142,37,141}
\definecolor{CambridgeSuperLightGreen}{RGB}{217, 234, 211}
\definecolor{CambridgeDarkLightBlue}{RGB}{84,138,182}
\definecolor{CambridgeLightBlue}{RGB}{106,173,228}
\definecolor{CambridgeLightOrange}{RGB}{239,189,71}
\definecolor{CambridgeLightGreen}{RGB}{168,180,0}
\definecolor{CambridgeLightPurple}{RGB}{181,147,155}
\definecolor{CambridgeLightTeal}{RGB}{163,193,173}
\definecolor{CambridgeCoreBlue}{RGB}{0,115,207}
\definecolor{CambridgeCoreOrange}{RGB}{227,114,34}
\definecolor{CambridgeCoreBurgundy}{RGB}{144,27,59}
\definecolor{CambridgeCoreGreen}{RGB}{88,166,24}
\definecolor{CambridgeCorePurple}{RGB}{142,37,141}
\definecolor{CambridgeCoreTeal}{RGB}{0,179,190}
\definecolor{CambridgeDarkBlue}{RGB}{0,62,114}
\definecolor{CambridgeDarkOrange}{RGB}{200,78,0}
\definecolor{CambridgeDarkGreen}{RGB}{67,81,37}
\definecolor{CambridgeDarkPurple}{RGB}{65,45,93}
\definecolor{CambridgeDarkTeal}{RGB}{21,101,112}
\definecolor{CambridgeWhitePrint}{RGB}{255,255,255}
\definecolor{CambridgeBlackPrint}{RGB}{0,0,0}
\definecolor{CambridgeYellow}{RGB}{220,20,6}

\definecolor{mygrey}{RGB}{170, 170, 170}
\definecolor{mylightgrey}{RGB}{230, 230, 230}

\usepackage{tikz}
\usetikzlibrary{automata, positioning, arrows.meta, decorations.markings, calc, fit}
\tikzset{
    middle arrow/.style={
        decoration={markings,
            mark=at position #1 with {\arrow{>}}}, 
        postaction={decorate},
        thick
    }
}

\usepackage{subcaption}
\usepackage{caption}

\usepackage{pgfplots}
\usepgfplotslibrary{groupplots}
\pgfplotsset{compat=1.18}

\usepackage{float}

\newcommand{\mcopi}{MCES}

\title{Exploring Starts Are Not Enough: \\ Counterexamples and a Fix for \\ Monte Carlo Exploring Starts}

\author{%
  Octave Oliviers \\
  Department of Engineering\\
  University of Cambridge\\
  Cambridge, UK \\
  \texttt{ofao2@cam.ac.uk} \\
   \And
  Glenn Vinnicombe \\
    Department of Engineering\\
  University of Cambridge\\
  Cambridge, UK \\
  \texttt{gv103@cam.ac.uk}
}

\begin{document}

\maketitle

\begin{abstract}
The asymptotic behaviour of Monte Carlo Exploring Starts (MCES) is a long-standing open question in reinforcement learning, even in the tabular setting.
We investigated the convergence properties of tabular MCES by constructing examples in which the algorithm converges to suboptimal solutions.
This paper presents new counterexamples for both initial-visit and first-visit MCES and gives a convergence-restoring modification for the initial-visit case.
We show that stable suboptimal solutions may exist for initial-visit MCES with sample-average updates even when greedy actions are updated more often than non-greedy actions on average.
However, by scaling learning rates inversely to update frequencies on a state-by-state basis, convergence to optimality is guaranteed.
Unlike previous uniformisation methods, this modification is applicable to large-scale problems that require approximating the estimated value function.
We then extend the example to show that sample-average first-visit MCES may also converge to suboptimal solutions.
This largely settles a fundamental open problem and shows that exploring starts alone do not guarantee convergence to optimality.
More broadly, these results highlight that convergence depends critically on the relative size and frequency of updates applied to different actions,
making the choice of learning rates and the balance between exploration and exploitation central to the analysis of MCES and the implementation of scalable Monte Carlo control methods.
\end{abstract}


\section{Introduction}

Reinforcement learning aims to train an agent to behave optimally in a certain environment.
In this paper, we model this interaction as a finite Markov decision process (MDP) evolving over discrete steps.
At time \(t\), the agent observes a state \(s_t \in \mathcal S\) and selects an action \(a_t \in \mathcal A(s_t)\) according to a policy \(\pi(a_t | s_t)\).
The environment then produces a reward \(r_t\) and transitions to a new state \(s_{t+1}\) according to its dynamics $p(s_{t+1}, r_t | s_t, a_t)$.

For a policy \(\pi\), the action-value function $q_\pi$ assigns to each pair $(s,a)$ the expected discounted return obtained by taking action \(a\) in state \(s\), and then following policy \(\pi\):
\[
q_\pi(s,a)
=
\mathbb E_{\pi} \left[
\sum_{t=0}^{\infty} \gamma^t r_t
\;\middle|\;
s_0 = s,\
a_0 = a
\right].
\]
Under standard finite-MDP assumptions,
there exists an optimal policy $\pi_*$ that satisfies $q_{\pi_*}(s,a) = \max_\pi q_\pi(s,a)$ for all $(s,a)$.
Moreover, $\pi_*$ may be chosen to be deterministic and greedy with respect to $q_{\pi_*}$, meaning
$\pi_*(s) \in \arg\max_{a \in \mathcal A(s)} q_{\pi_*}(s,a)$ for all $s$.


Monte Carlo Exploring Starts (MCES) maintains an action-value estimate $q_k$
and a policy $\pi_k$ greedy with respect to $q_k$.
At iteration $k$, the algorithm samples an initial state-action pair according to a distribution $\sigma_k$
and simulates an episode starting from that pair:
\[
\mathcal E_k
:=
\bigg\{
s_0,a_0,r_0,s_1,a_1,r_1,\ldots : 
(s_0,a_0) \sim \sigma_k, \
(s_{t+1},r_t) \sim p(\cdot, \cdot| s_t, a_t), \
a_t = \pi_k(s_t), \
\forall \, t \ge 1
\bigg\} .
\]
It then updates the estimate $q_k$ at pairs $(s,a)$ appearing in $\mathcal E_k$ as
\begin{align}
\label{eq: mcopi def}
q_{k+1}(s,a) = q_k(s,a) + \eta_k(s,a) \big( \tilde q_k(s,a) - q_k(s,a) \big) ,
\end{align}
while \(q_{k+1}(s,a) = q_k(s,a)\) for non-updated pairs.
Here,
$\eta_k(s,a)$ is a learning rate satisfying the Robbins--Monro conditions $\sum_{k} \eta_k(s,a) = \infty$ and $\sum_{k} \eta_k^2(s,a) < \infty$ for each $(s,a)$,
and
$\tilde q_k(s,a)$ is the observed return from that pair onwards:
\begin{align}
\label{eq: episode return}
\tilde q_k(s_t,a_t) = \sum_{i \geq 0} \gamma^{i} r_{t+i} 
\qquad
\forall \, t \ge 0 .
\end{align}
%
%
%
%
%
We consider two update rules.
The \emph{initial-visit} (IV) method only updates $q_k$ at the initial pair $(s_0,a_0)$.
The \emph{first-visit} (FV) method updates each state-action pair in the episode, using the return from its first occurrence in the episode.
We also consider different choices of learning rates.
The \emph{sample-average} (SA) rule sets $\eta_k(s,a) = 1/n_k(s,a)$, where $n_k(s,a)$ is the number of times $(s,a)$ has been updated up to and including time $k$. Under this rule, $q_{k}(s,a)$ is the sample average of the observed returns for $(s,a)$ up to time $k$.
Alternatively, the learning rates $\{\eta_k\}_{k \ge 0}$ can be independent of $(s,a)$ and just form a scalar sequence.

To ensure sufficient exploration,
the \emph{exploring starts} 
condition requires that $\sigma_k(s,a) > 0$ for all $k$ and all $(s,a)$.
In order to ensure all state-action pairs are visited infinitely often, we introduce a stricter condition:
there exists a constant $\sigma_{\min}>0$ such that $\sigma_k(s,a) \ge \sigma_{\min}$ for all $k$ and all $(s,a)$.

For both IV and FV, the observed return $\tilde q_k$ is an unbiased estimate of $q_{\pi_k}$~\cite{Singh1996ReinforcementTraces}.
So, defining the zero-mean perturbation $v_k(s,a) := \tilde q_k(s,a) - q_{\pi_k}(s,a)$,
equation \eqref{eq: mcopi def} can be written as
\begin{align}
\label{eq: mcopi def noisy}
q_{k+1}(s,a) = q_k(s,a) + \eta_k(s,a) \big( q_{\pi_k}(s,a) - q_k(s,a) + v_k(s,a) \big) .
\end{align}
MCES is thus a stochastic approximation to the policy evaluation map associated with the current greedy policy.
The dynamics are fixed as long as $q_k$ remains in a region where the same policy is greedy. For a policy $\pi$, define this region by
\begin{align}
\label{eq: greedy action values}
    Q(\pi) \coloneqq \big\{ q  : q(s, \pi(s)) = \max_{a \in \mathcal A(s)} q(s,a), \ \forall \, s \in \mathcal S \big\} .
\end{align}
%
%
%
Discussing the FV/SA variant of MCES,
Sutton
writes:
``It is hard to imagine any RL method simpler or
more likely to converge than this $\ldots$
While this simplest case remains open we
are unlikely to make progress on any control method for $\lambda > 0$.''~\cite[p.~13]{Sutton1999OpenLearning}
(Here, $\lambda$ refers to the parameter in Q($\lambda$), TD($\lambda$) etc.~\cite{Sutton1988LearningDifferences}.)
Sutton and Barto later reinforce the importance of this problem by labelling it as ``one of the most fundamental open theoretical questions in reinforcement learning.''~\cite[p.~99]{Sutton2018ReinforcementIntroduction}




\subsection{Previous work}

The Robbins--Monro conditions alone do not prevent IV MCES from converging to suboptimal solutions~\cite{Bertsekas1996Neuro-DynamicProgramming, Wang2022OnLearning}.
In these counterexamples, the sampling distribution is strongly biased towards non-greedy actions (i.e. $\sigma_k(s,a_1)>\sigma_k(s,a_1)$ when $q_k(s,a_1)<q_k(s,a_2)$), as described in \autoref{app: non greedy}.
In contrast, with uniform sampling distributions over all state-action pairs and scalar learning rates, IV MCES converges almost surely to optimality~\cite{Tsitsiklis2002OnIteration, Chen2018OnProblem, Liu2021OnStarts}.
For large problems, however, sampling uniformly over the state space is not feasible.

For completeness, we note that MCES also converges to optimality when the MDP exhibits a feedforward structure \citep{Wang2022OnLearning, Lubars2021OptimisticStructure}, and that 
more complex multi-step algorithms manage to recover contraction arguments 
by combining several multi-step approximations~\citep{Munos2016SafeLearning},
or using a look-ahead mechanism~\citep{Winnicki2023OnEvaluation}.
We focus, however, on the behaviour of tabular MCES for general MDPs as even this simple case remains an open question.


\subsection{Contributions}

We make three contributions.
First, we construct an MDP in which IV/SA MCES does not converge to $q_{\pi_*}$ even when sampling distributions are biased towards greedy actions.
Second, we prove that IV MCES converges to $q_{\pi_*}$ when scaling learning rates inversely to update frequencies on a state-by-state basis.
This modification works in large-scale environments as it does not require uniform sampling over the state space.
Unlike earlier work, the proof derives convergence by analysing the evolution of the policy $\pi_k$ instead of the estimate $q_k$.
Third, we show that FV/SA MCES can converge to suboptimal solutions when noise is averaged out.
Thus any convergence proof for FV/SA MCES would need to show that the stochastic sampled paths always escape the mean-field attractors.
Our simulations suggest that this does not happen.
For all simulation horizons we could test, we observed fully sampled runs getting trapped in the same suboptimal basin.
These results provide strong evidence against the global convergence of
FV/SA MCES in general MDPs.

\section{Counterexample for initial-visit MCES with sample average updates}
\label{sec: counterexample iv mces}

\label{sec: 3 state sampling strat}
Consider the MDP in \autoref{fig: explain loop on-policy bias MDP} with $\gamma = 0.9$.
Each state has two actions: "move" yields a reward of one while "stay" yields no reward. The policy $\pi_*$ chooses "move" everywhere, denoted by $a_*$, while $\bar \pi$ chooses "stay", denoted by $\bar a$.

\input{figures/explain-loop}


The following sampling strategy forces IV MCES to cycle indefinitely between the suboptimal policies $\pol_1$, $\pol_2$, $\pol_3$, $\pol_4$, $\pol_5$ and $\pol_6$ for every initial condition in the region $\qset_0$, which 
is the intersection between the regions of these six policies and the shaded region in \autoref{fig: explain loop on-policy bias MDP}:
\begin{enumerate}
    \item  \label{pg: sampling strat 3 states}
    When $q_k \in \region(\pol_1)$, start each episode in $(s_2, \best{a})$ or $(s_3, \best{a})$ with a probability of 0.4, or in $(s_2, \worst{a})$ with a probability of 0.2.
    \mcopi{} is then expected to cross the boundary in state $s_3$ and transition to policy $\pol_2$.

    \item
    When $q_k \in \region(\pol_2)$, start each episode in $(s_1, \worst{a})$. 
    MCES will cross the boundary in $s_1$ and transition to $\pol_3$.

    \item
    When $q_k \in \region(\pol_3)$, start each episode in $(s_1, \best{a})$ or $(s_2, \best{a})$ with a probability of 0.4, or in $(s_1, \worst{a})$ with a probability of 0.2.
    \mcopi{} is then expected to cross the boundary in $s_2$ and transition to $\pol_4$.

    \item
    When $q_\kpz \in \region(\pol_4)$, start each episode in $(s_3, \worst{a})$. 
    MCES will cross the boundary in $s_3$ and transition to $\pol_5$.

    \item
    When $q_\kpz \in \region(\pol_5)$, start each episode in $(s_1, \best{a})$ or $(s_3, \best{a})$ with a probability of 0.4, or in $(s_3, \worst{a})$ with a probability of 0.2.
    \mcopi{} is then expected to cross the boundary in $s_1$ and transition to $\pol_6$.

    \item
    When $q_\kpz \in \region(\pol_6)$, start each episode in $(s_2, \worst{a})$. 
    \mcopi{} will cross the boundary in $s_2$ and transition to $\pol_1$.
\end{enumerate}
\label{enum: sampling greedy initial visit}

The cycle is robust to noise caused by sampling the initial state-action pairs in steps 1, 3 and 5.
For example, if $(s_2, \worst{a})$ is updated several times in a row during step 1, the solution might switch to policy $\pol_6$ instead of $\pol_2$.
Nevertheless, the update distribution under $\pol_6$ directly corrects this by updating $(s_2, \worst{a})$ and driving the solution back to $\pol_1$.
Thus, the sampling strategy guarantees that, for every initial condition in $\mathcal Q_0$, the greedy policy $\pi_k$ contains at least one suboptimal action but never all three.
As a result, 
solutions cycle indefinitely with potential pairwise repeats between $\pi_1$ and $\pi_6$, $\pi_3$ and $\pi_2$, or $\pi_4$ and $\pi_5$ but without ever visiting the policies $\best{\pol}$ or $\worst{\pol}$.
This behaviour does not rely on the exact probabilities above.
It persists when increasing the sampling bias towards greedy actions,
for instance when replacing the 2-to-1 sampling ratio of $(s_2, \best{a})$ to $(s_2, \worst{a})$ during step 1 with an epsilon-greedy strategy for small epsilon.

We incorporate exploring starts 
by initialising each episode using the sampling strategy above with a probability of 0.9 
or by uniformly sampling a state-action pair with a probability of 0.1.
Since the wrong state-action pairs are now sampled at the wrong time infinitely often,
solutions can visit the policies $\best{\pol}$ or $\worst{\pol}$.
When they do, we sample the initial state-action pair uniformly.
Additionally, we introduce stochastic transitions by terminating each episode with a probability of 0.1 after each action.
The value $q_{\pol_k}(s,a)$ is then always over- or underestimated, which yields steps that potentially drive solutions out of $\qset_0$.

\input{figures/fig_es3}

Despite the variability of initial state-action pairs and the noisy action-value estimates, IV MCES still converges to a suboptimal solution when using SA learning rates.
For instance,
$99.6\%$ of solutions converge to a point where $q(s, a_*) = q(s, \bar a) \approx 1.615$ for all states
when starting in
$q_0(s_1,a_*) = q_0(s_2,a_*) = q_0(s_3,\bar a) = 2$
and
$q_0(s_1,\bar a) = q_0(s_2, \bar a) = q_0(s_3,a_*) = 1$
with $n_1(s,a)=1$ for all state-action pairs (\autoref{fig: 3s IV ES ST}).
The remaining $0.4\%$ escape to $q_{\pi_*}$.




The reason solutions converge to the boundary despite frequently visiting the optimal policy $\pi_*$ is as follows.
Consider a solution in the region $Q(\pi_*)$ near the boundary point $1.615$.
So $q_k(s, a_*) > q_k(s, \bar a)$ 
for every state,
and $q_k(s,a) < \qopt(s,a)$ for every state-action pair.
There are two possible updates.
Updating $(s, a_*)$ moves $q_k$ deeper into $Q(\pi_*)$ on average as $q_k(s, a_*)$ increases compared to $q_k(s, \bar a)$.
In contrast, updating $(s, \bar a)$ guides the solution towards the boundary of $Q(\pi_*)$ as $q_k(s, \bar a)$ catches up to $q_k(s, a_*)$. If the learning rate is large enough, 
$\bar a$ may even become greedy in $s$.
Thus, solutions consistently leave the optimal region only if 
updates to $(s, \bar a)$ receive larger learning rates or are more frequent on average than updates to $(s, a_*)$.
Precisely this behaviour is induced by the sampling strategy, the initial-visit updates and the sample-average learning rates.
Specifically, the sampling strategy 
combined with initial-visit updates implies that, for every state, $a_*$ is updated more frequently than $\bar a$ when solutions cycle between the policies $\pol_1$, $\pol_2$, $\pol_3$, $\pol_4$, $\pol_5$ and $\pol_6$.
Therefore, $1/n_k(s, \best{a})$ decreases faster than $1/n_k(s, \worst{a})$ for each state.
As a result, when noise pushes solutions into the optimal region $Q(\pi_*)$, all actions are updated equally frequently but updates to $(s, \worst{a})$ are larger than updates to $(s, \best{a})$,
which drives solutions out of $Q(\pi_*)$.
In other words, the sampling strategy, the initial-visit updates and the learning rates effectively bias updates towards suboptimal actions when $\pol_k = \best{\pol}$ even though greedy actions are updated more frequently. 
As discussed above, this 
allows solutions to consistently jump in and out of the optimal region, creating a stable suboptimal boundary region when the learning rates decrease.
%

Overall,
this example shows that IV/SA MCES may converge to suboptimal solutions even when, in each state, the greedy action is updated at least as frequently as any other action.
The asymptotic behaviour is determined by the interplay
between learning rates and within-state update frequencies, rather than either factor alone.
This suggests a remedy: scale learning rates to offset non-uniform update
frequencies in each state and recover convergence to optimality.

\section{Fix for initial-visit MCES}

Starting from an initial $q_0$,
IV MCES repeats the following steps for $k=0,1,\ldots,\infty$:
\begin{enumerate}
    \item Select a deterministic policy $\pi_k(s) \in \arg\max_a q_k(s,a)$ for all $s$.

    \item Sample an initial state-action pair $(s_k, a_k) \sim \sigma_k$.

    \item Update the estimate $q_k$ as
    \begin{align}
    \label{eq: update iv mcopi}
    q_{k+1}(s,a)
    =
    \begin{cases}
    q_k(s,a) + \eta_k(s,a) \left(\tilde q_k(s,a) - q_k(s,a) \right), & (s,a)=(s_k,a_k) \\
    q_k(s,a) & \text{otherwise .}
    \end{cases}
    \end{align}
    where $\tilde q_k$ is the episode return defined in \eqref{eq: episode return}.
\end{enumerate}

Section~\ref{sec: counterexample iv mces}, together with Appendix~\ref{app: non greedy} and Example~5.12 in \cite{Bertsekas1996Neuro-DynamicProgramming}, shows that IV MCES may converge to suboptimal solutions when the sampling distributions are biased towards either greedy or non-greedy actions.
This section presents a way to compensate for such non-uniform sampling by appropriately normalizing the learning rates, and thereby restoring convergence to optimality.
The approach relies on the chain rule decomposition of the sampling distribution:
\begin{align}
\label{eq: chain rule}
\sigma_k(s,a) = f_k(s) \, g_k(a | s) .
\end{align}

To formalise this correction, 
let $\mathcal F_k$ denote the information available after computing $q_k$,
but before selecting policy $\pi_k$, sampling the initial pair $(s_k, a_k)$, and generating the episode $\mathcal E_k$.

\begin{theorem}
\label{thm: iv mcopi with 1/g scaling}
Assume
\begin{enumerate}[label=(\arabic*)]
    \item \textbf{Discounted or episodic MDP.}
    Either the discount factor satisfies $\gamma \in [0, 1)$ or each episode of the MDP almost surely reaches a terminal state.

    \item \textbf{Uniform tie-breaking of greedy policies.}
    There exists a constant \(\beta>0\) such that,
    if several policies are greedy with respect to $q_k$,
    \begin{equation}
    \label{eq:greedy-pol-tie-breaking}
    \mathbb P(\pi_k(a | s) = 1 \mid \mathcal F_k)\ge \beta,
        \qquad 
        \forall \, k \ge 0, \
        \forall \, s \in \sset, \
        \forall \, a \in \arg\max_{\mathcal A(s)} q_k(s,\cdot) .
    \end{equation}

    \item \textbf{Strict exploring starts.}
    The sampling distributions $\{\sigma_k\}_{k \ge 0}$ satisfy
    \begin{align}
    \label{eq: strict exploring starts}
        \sigma_k(s,a) \ge \sigma_{\min}
        \qquad 
        \forall \, k \ge 0, \
        \forall \, s \in \sset, \
        \forall \, a \in \aset(s) ,
    \end{align}
    for some constant $\sigma_{\min} \in (0,\infty)$.

    \item \textbf{Scalar, Robbins--Monro and comparable learning rates.} 
    The sequence $\{\alpha_k\}_{k \ge 0}$ forms a scalar non-increasing sequence: $\alpha_{k+1} \leq \alpha_k$ for every $k \ge 0$.
    It satisfies
    \begin{align}
    \label{eq: robbins monro conditions}
        \sum_{k =0}^\infty \alpha_k = \infty , 
        \qquad
        \sum_{k =0}^\infty \alpha_k^2 < \infty ,
    \end{align}
    and there exist a time $K_\alpha < \infty$ and a constant $\rho_\alpha \in(0,1]$ such that
    \begin{align}
    \label{eq: comparability of learning rates}
    \alpha_{k+i} \ge \rho_\alpha \alpha_k ,
    \qquad
    \forall \, k \ge K_\alpha, \
    \forall \, i \in \{0, \dots, \left\lfloor 1 / \alpha_k \right\rfloor\} .
    \end{align}
    Given that condition \eqref{eq: robbins monro conditions} requires that $\alpha_k \to 0$ and we study the asymptotic behaviour of MCES, we assume, without loss of generality, that $\alpha_k \in (0, 1]$.
\end{enumerate}
Then IV MCES converges almost surely to $q_{\pi_*}$ when using learning rates
\begin{align}
\label{eq: normalised learning rates}
\eta_k(s,a) := \dfrac{\alpha_k}{g_k(a|s)} .
\end{align}
\end{theorem}

The scaling in \eqref{eq: normalised learning rates} is a local analogue of SA learning rates.
While $1/n_k(s,a)$ gives a greater weight to actions that have been sampled less often in the past,
the factor $1/g_k(a|s)$ gives a greater weight to actions which are less likely to be sampled at the current step.
The correction is thus local in time: it compensates for the current sampling bias rather than the historical sampling frequency.

The key idea in proving \autoref{thm: iv mcopi with 1/g scaling} is to analyse the sequence of policies $\{\pi_k\}_{k\ge 0}$ rather than the evolution of the estimates $\{q_k\}_{k\ge 0}$.
We find that the normalization in \eqref{eq: normalised learning rates}
causes the associated mean-field dynamics to generate monotonically improving policies
because it equalises the rate of change across actions on a state-by-state basis at every step.
This structure allows us to establish convergence of the original stochastic solutions. The complete proof is given in \autoref{app: proof of stochastic theorem}.


\begin{proof}[Proof sketch]
Equation \eqref{eq: update iv mcopi} with learning rates as in \eqref{eq: normalised learning rates} is equivalent to
\begin{align}
\nonumber
    q_{k+1}(s,a)
    &=
    q_k(s,a) + \alpha_k \left( \frac{\sigma_k(s,a)}{g_k(a|s)} \big( q_{\pi_k}(s,a) - q_k(s,a) \big) + w_k(s,a) \right)
    \\
\label{eq: iv mcopi in stochastic theorem}
    &=
    q_k(s,a) + \alpha_k \left( f_k(s) \big( q_{\pi_k}(s,a) - q_k(s,a) \big) + w_k(s,a) \right)
\end{align}
for every state-action pair.
Here, $w_k$ incorporates the perturbations caused by the episode simulation and the asynchronous updates of state-action pairs.
Given that $\tilde q_k$ is an unbiased estimate of $q_{\pi_k}$ under initial-visit updates, it follows that $\mathbb E[w_k(s,a) \mid \mathcal F_k] = 0$.

    
\textit{Step 1. Mean-field generates monotonically improving policies}

Consider the mean-field dynamics underlying \eqref{eq: iv mcopi in stochastic theorem}:
\begin{align}
\label{eq: deterministic mcopi}
    q_{k+1}(s,a)
    =
    q_k(s,a) + \alpha_k f_k(s) (q_{\pi_k}(s,a) - q_k(s,a)) .
\end{align}
Since the new policy $\pi_{k+1}$ will be greedy with respect to $q_{k+1}$, we have $q_{k+1}(s,\pi_{k+1}(s)) \ge q_{k+1}(s, \pi_k(s))$ for every state $s$.
Using \eqref{eq: deterministic mcopi}, this becomes
\begin{multline*}    
    q_k(s, \pi_{k+1}(s)) + \alpha_k 
    f_k(s)
    \big( q_{\pi_k}(s, \pi_{k+1}(s)) - q_k(s, \pi_{k+1}(s)) \big)
    \\
    \ge
    q_k(s, \pi_k(s)) + \alpha_k 
    f_k(s)
    \big( q_{\pi_k}(s, \pi_k(s)) - q_k(s, \pi_k(s)) \big) .
\end{multline*}
Rearranging the inequality
yields
\begin{align*}    
    q_{\pi_k}(s, \pi_{k+1}(s)) - q_{\pi_k}(s, \pi_k(s))
    \ge
    \dfrac{1 - \alpha_k f_k(s)}{\alpha_k f_k(s)}
    \big( q_k(s, \pi_k(s)) - q_k(s, \pi_{k+1}(s)) \big) .
\end{align*}
Since $\pi_k$ is greedy with respect to $q_k$ and $0 < \alpha_k f_k(s) \le 1$ for sufficiently large $k$, 
it follows that
\begin{align*}
    q_{\pi_k}(s, \pi_{k+1}(s)) - q_{\pi_k}(s, \pi_k(s)) \ge 0
\end{align*}
for every state $s$.
As a result, the Policy Improvement Theorem guarantees that the new policy $\pi_{k+1}$ is at least as good as policy $\pi_k$~\cite{Sutton2018ReinforcementIntroduction}.

\textit{Step 2. Noise cannot consistently prevent policy improvement}

We enforce improving transitions by maintaining a positive distance to "bad" regions using three events with fixed non-zero probabilities:
\begin{enumerate}
\item \emph{Seed}: Force the distance to $O(\alpha_k)$, with probability $p_s>0$.
\item \emph{Grow}: Bring it to $O(1)$ using the mean-field drift, with probability $p_g>0$.
\item \emph{Lock-in}: Keep it strictly positive until the next transition, with probability $p_l>0$.
\end{enumerate}
Since, at any time, the next policy will be better with probability at least $p_s p_g p_l > 0$,
the probability of never reaching and locking-in to the optimal region grows as $\big(1 - (p_s p_g p_l)^{|\mathcal P|} \big)^n$, where $\mathcal P$ is the set of deterministic policies.
This probability tends to zero as $n$ tends to infinity because the number of deterministic policies is finite.
As a result, the complement event, i.e. locking into the optimal region and $q_k \to q_{\pi_*}$, occurs eventually with probability 1.
\end{proof}





Since \(g_k(\cdot \mid s)\) is a distribution over \(\mathcal A(s)\), its values scale as \(1/|\mathcal A(s)|\). Hence, \(1/g_k(a\mid s)\) is typically larger in states with more available actions.
If necessary, dividing the learning rates in \eqref{eq: normalised learning rates} by \(|\mathcal A(s)|\) cancels out this effect.
Since this additional factor is constant across actions within a state, it does not affect the argument. It only introduces extra state-dependent scaling factors in the proof.

Earlier work recovers convergence to optimality by scaling the learning rates with $1/\sigma_k$~\cite{Tsitsiklis2002OnIteration, Liu2021OnStarts}.
However, this requires knowing the sampling distribution over \textit{all} state-action pairs.
In contrast, the learning rates in \eqref{eq: normalised learning rates} 
only require the conditional sampling probability of the
initial action given the initial state.
Thus, convergence can still be recovered when the initial-state distribution is unknown, provided the action distribution within each state is known. This is often the case in practice, for example when episodes are initialized using an epsilon-greedy policy.
This weaker condition is more practical for real-world applications,
in particular when the state-space is large. For instance, chess has
$O(10^{45})$ different board layouts but only $O(10^1)$ legal moves per layout.
Keeping track of a distribution over the joint state-action space is not feasible, whereas representing the action distribution within a given state is straightforward.

Earlier work recovers convergence to optimality by making updates effectively uniform over all state-action pairs, 
whereas \autoref{thm: iv mcopi with 1/g scaling} shows that uniform updates over the actions within each state are sufficient.
The earlier work is based on the theorem of Tsitsiklis in~\cite{Tsitsiklis2002OnIteration}, which uses commutativity in a central way.
Since this technique does not work when updates are only uniform over the actions within each state, a different mathematical approach is required.


\section{Counterexample for first-visit MCES with sample average updates}\label{sec:fv}

Under first-visits, the state-action pairs updated after an episode depend on the transition dynamics induced by the greedy policy on the MDP.
FV MCES is thus harder to analyse than IV MCES and counterexamples are harder to construct.
This section presents such a counterexample: a cyclic MDP with $13$ states
generalising \autoref{fig: explain loop on-policy bias MDP}, on which FV/SA MCES frequently converges to a suboptimal solution.
Each state has two actions:
$a_*$ gives reward $1$ and moves to the next state with probability $\tau$ or terminates the episode with probability $1-\tau$; action $\bar a$ yields no reward and stays in the same state.

%

We found the example by first analysing FV/SA MCES without noise, namely when $q_k(s,a)$ is updated towards the exact value $q_{\pi_k}(s,a)$, and $n_k(s,a)$ is incremented by the probability of updating $(s,a)$ after an episode simulated with policy $\pi_k$.
For $N=3$ this deterministic process does not converge to a suboptimal solution. 
For $N=4$ it can oscillate indefinitely between suboptimal policies, but 
this behaviour is fragile as it requires that the greedy policy changes in two states simultaneously.
For $N \ge 5$ the same type of oscillation can proceed through one-state policy changes;  the oscillation alternates between policies with one and two "stay" actions, as depicted in \autoref{fig: long MDP}. 
This suboptimal behaviour persists under the noise of the fully sampled FV/SA MCES process for sufficiently late starts, becoming increasingly robust to fluctuations early in the simulation as $N$ grows. We describe below the case N=13.


The deterministic solutions oscillate indefinitely between the 26 policies when initialising episodes as follows.
At each time step, with probability $\epsilon_{\mathrm{ES}}$ sample the initial state–action pair uniformly amongst the 26 pairs;
otherwise, choose it according to the current greedy policy.
If that policy has a single "stay" action at $s_i$, start the episode in $(s_{i+6}, \bar a)$.
If the policy has adjacent "stay" actions at $s_i$ and $s_{i+1}$, start in $(s_{i-1}, \bar a)$ with probability $\lambda$ or in $(s_{i+1}, a_*)$ with probability $(1-\lambda)$.
All state indices are to be taken modulo 13.
Under this initialisation scheme, 
\autoref{fig:spiral} shows that the deterministic solutions spiral into a fixed point on the the boundary.
This behaviour is robust as a large set of initial conditions produce this spiral.
We derive the value $\widehat q$ in Appendix~\ref{sec:balance-general},
and show that the subsequence $\{q_{k + 26 i}\}_{i \ge 0}$ moves along a straight line towards that fixed point in Appendix~\ref{sec:linear-interpolation}.
As a result, it is guaranteed that the deterministic solutions visits only these 26 policies.

\input{figures/spiral}

Fully sampled FV/SA MCES solutions are extremely noisy, due to both sample returns and exploring starts.
As a result, they frequently visit off-cycle policies.
We then initialise episodes to 
guide the process back to an on-cycle policy as follows.
On visiting the optimal all-"move" policy,
we start episodes in the "stay" action of the state for which the action-values are closest to the boundary.
The first visit structure ensures that "move" is then also updated in that, and potentially downstream, states, 
and also that "move" actions are updated more often than "stay" actions overall. The $1/n_k(s,a)$ weighting then results in solutions being pushed towards and eventually across the boundary, creating a "stay" state.
In all other off-cycle policies,
we first consider all states where "stay" is greedy and keep only those with the fewest neighbours that also prefer "stay". Among that subset, we choose the state whose action values are closest to the decision boundary, then initialize the episode in the "stay" action of that state.
This eventually results in solutions crossing the boundary in the opposite direction, reducing the number of "stay" states until only a singleton or adjacent pair remain.   
Exploring starts remain active during the recovery.
The sampling strategy is thus completely determined by the current $q_k$ (as would be the case for an $\epsilon$-greedy policy, for example) with a uniform lower bound on the probability of sampling any state-action pair at each step.

We simulated $40,000$ solutions of FV/SA MCES with $\gamma=0.9999$, $\tau\approx 0.5123$ and $\epsilon_{ES}=10^{-4}$ for $10^9$ steps for a range of initial counts $n_k(s,\bar a)=M$, $n_k(s,a_*)=2.01M$ with $M$ from  $3000$ to $12000$ and found that all but a proportion of approximately $(2200/M)^{5.4}$ remained trapped in a shrinking strip around the boundary significantly below the optimal value. In each run, for approximately $14.9\%$ of steps the greedy policy was the all-"move" policy $\pi_*$ and for approximately $1.5\%$ of steps it was another off-cycle policy. Thus for a large majority of steps the system is in a cycle phase.       
For a representative sample of runs which hadn't escaped at $10^9$ steps we continued the simulation up to $10^{12}$ steps and none escaped. We then ran $8$ simulations from $M=10^5$ and all remained trapped at $10^{12}$ steps. That is, all trials that survived $10^9$ steps went on to survive $10^{12}$ steps. 
\autoref{fig:fvsasim} shows representative behaviour from a start at $M=1000$, demonstrating an initial lock in followed by the long $1/k$ tail.
Altogether, these simulations provide strong evidence 
of high probability non-convergence to optimality from late starts of FV/SA MCES for this MDP; i.e.
that, for large $M$, starts from $n_k(s,\bar a)=M$ escape from the stable suboptimal basin with probability $O(M^{-K})\rightarrow 0$ for $K\approx 5.4$. 

\input{figures/fv-sims}

\section{Conclusion}

We have presented two new counterexamples showing that both initial-visit and first-visit MCES can converge to suboptimal solutions, 
and proposed a practical fix that guarantees convergence to optimality in the initial-visit case.

Previous counterexamples for initial-visit MCES rely on sampling distributions that heavily favour non-greedy actions at every step.
Our example shows that stable suboptimal solutions exist even when greedy actions are updated more frequently, as typically occurs during the exploitation phase.
The suboptimal behaviour arose because the sample-average learning rates introduced an effective bias towards non-greedy actions when solutions were in the optimal region.

We then presented a practical learning-rate scheme that restores convergence to optimality despite updating actions at different frequencies. 
The key idea is to scale learning rates inversely with update frequencies on a state-by-state basis, so that the effective update size is uniform across actions within each state. Under this correction, the underlying mean-field generates monotonically improving policies. 
Since the stochastic perturbations cannot consistently prevent this improvement, initial-visit MCES eventually converges to the optimal policy and its action values.

Finally, we constructed a 13-state MDP in which first-visit MCES with sample-average updates converges to a suboptimal solution.
This example establishes that any convergence proof for first-visit MCES would need to show that solutions always escape suboptimal mean-field attractors.
In our simulations this occurs with sharply decreasing frequency for later starts, providing strong evidence that exploring starts alone do not guarantee convergence of Monte Carlo control methods to optimality.

Overall, these results highlight that
the asymptotic behaviour of MCES depends critically on the interplay between learning rates and sampling distributions.
By controlling the frequency and magnitude of updates to each action, these parameters determine the effective update size.
Initial experiments (not presented here) suggest that MCES converges to optimality when, within each state, the effective update to the greedy action is at least as large as that to any other action.
Establishing this condition formally is an important direction for future work
as it may provide a foundation for analysing scalable MCES variants.

\clearpage
\bibliographystyle{unsrt} 
\bibliography{References/references}

\appendix

\newpage
\section{Counterexample for initial-visit MCES with non-greedy bias}
\label{app: non greedy}

Bertsekas and Tsitsiklis show in \cite[Example~5.12]{Bertsekas1996Neuro-DynamicProgramming} that IV MCES can converge to suboptimal solutions when greedy actions are updated less frequently than non-greedy actions.
The single-state MDP in \autoref{fig: experiment off-policy bias MDP} is the simplest such example.

Episodes are initialised to favour the non-greedy action,
meaning
when $q_k(s,\pi_*(s)) \ge q_k(s,\bar\pi(s))$
the initial state-action pair is $(s,\pi_*(s))$ with a small probability $\epsilon$ and $(s,\bar\pi(s))$ with probability $1-\epsilon$,
with the probabilities reversed otherwise.

\autoref{fig: experiment for off-policy bias} shows simulations with $\epsilon = 0.1$. The streamlines indicate the mean-field flow.
Solutions starting to the right of the separatrix typically converge to $q_{\pi_*}$, whereas those starting to the left leave the optimal region $Q(\pi_*)$ and converge to a boundary point.


\input{figures/counter-MDP}

\input{figures/off-policy-simulation1}

\section{Proof of \autoref{thm: iv mcopi with 1/g scaling}}
\label{app: proof of stochastic theorem}





Define the perturbation $v_k := \tilde q_k - q_{\pi_k}$ where $\tilde q_k$ is the episode return defined in \eqref{eq: episode return},
as well as the indicator $u_k(s, a) := \mathbf{1}_{\{(s,a)=(s_k,a_k)\}}$ that captures the randomness of asynchronous updates.
Combine these two sources of noise into a single stochastic perturbation $w_k$ defined as
\begin{align}
\label{eq: noise definition appendix}
w_k(s,a) 
&:=
\frac{1}{g_k(a | s)}
\Big( u_k(s,a) v_k(s,a)
+
\big( u_k(s,a) - \sigma_k(s,a) \big) \big( q_{\pi_k}(s,a) - q_k(s,a) \big) \Big) .
\end{align}
Equation \eqref{eq: update iv mcopi} with learning rates as in \eqref{eq: normalised learning rates} is then equivalent, for every state-action pair, to
\begin{align}
\nonumber
q_{k+1}(s,a)
&=
q_k(s,a) + \alpha_k \frac{u_k(s,a)}{g_k(a | s)} \big( q_{\pi_k}(s,a) + v_k(s,a) - q_k(s,a) \big)
\\
\label{eq: stochastic rep of normalised}
    &=
    q_k(s,a) + \alpha_k \big( f_k(s) ( q_{\pi_k}(s,a) - q_k(s,a) ) + w_k(s,a) \big) 
\end{align}
where $\pi_k$ is a deterministic policy that is greedy with respect to $q_k$.

\autoref{thm: iv mcopi with 1/g scaling} is an adaptation of Proposition~20 in \cite{Oliviers2026ConvergenceUpdates}. The event construction and policy-improvement argument are unchanged: 
stochastic perturbations cannot consistently undermine the policy improvement driven by the mean-field dynamics,
so $\pi_k$ eventually stabilises at $\pi_*$ and $q_k \to q_{\pi_*}$.

The only modification is the normalized learning rate $\alpha_k/g_k(a\mid s)$ in place of $\alpha_k$. Since $g_k(a\mid s) \in (0,1]$, this rescaling is positive and at least as large as $\alpha_k$, so the induction in Lemma~15 is unaffected. Lemmas~17 and~18 require only that $w_k$ is a martingale difference sequence with bounded second moment; this is verified in the lemma below.




\begin{lemma}
\label{thm: properties of noise}
There exists a constant $c_w \in [0, \infty)$ 
such that
\begin{align}
\label{eq: condition on combined noise for proof}
\mathbb E[w_k(s,a) \mid \mathcal F_k] = 0,
\ \ \
\mathbb E[w_k^2(s,a) \mid \mathcal F_k] \le c_w(1+\|q_k\|^2),
\ \ \
\forall \, k \ge 0, \,
\forall \, s \in \mathcal S, \,
\forall \, a \in \mathcal A(s) .
\end{align}
\end{lemma}
\begin{proof}
Fix a time $k\ge 0$, a state $s \in \mathcal S$ and an action $a \in \mathcal A(s)$.
Let $\mathcal F_k'$ represent the available information right after selecting the policy $\pi_k$ but before sampling the initial state-action pair $(s_k, a_k)$, so $\mathcal F_k \subseteq \mathcal F_k'$.

Under initial-visit updates, $\tilde q_k$ is an unbiased estimate of $q_{\pi_k}$ with uniformly bounded variance
\cite{Singh1996ReinforcementTraces}.
Thus there exists a constant $c_v \in [0,\infty)$ such that
\begin{align}
\label{eq: properties episode noise}
\mathbb E\!\left[v_k(s,a) \mid \mathcal F_k \right] = 0 ,
\qquad
\mathbb E\!\left[v_k^2(s,a) \mid \mathcal F_k \right] \le c_v,
\qquad
\forall \, k \ge 0 .
\end{align}
Further, 
$u_k(s,a) v_k(s,a) = 0$ on $\{u_k(s,a)=0\}$.
Hence,
\begin{align}
\label{eq: noise unbiased part 1}
\mathbb E\big[ u_k(s,a) v_k(s,a) \mid \mathcal F'_k \big]
=
\mathbb E \Big[ 
    \mathbf{1}_{\{u_k(s,a)=1\}} \ 
    \mathbb E \big[ v_k(s,a) \mid \mathcal F'_k, u_k(s,a) = 1 \big]
    \mid \mathcal F'_k \Big]
    = 0 .
\end{align}
Also,
$\E[u_k(s,a) \mid \mathcal F'_k] = \sigma_k(s,a)$ 
and
$\sigma_k(s,a)$ and $(q_{\pi_k}(s,a)-q_k(s,a))$ are $\mathcal F'_k$-measurable. It follows that
\begin{multline}
\mathbb E \!\left[ 
\big( u_k(s,a) - \sigma_k(s,a) \big) \big( q_{\pi_k}(s,a) - q_k(s,a) \big)
\mid \mathcal F'_k
\right]
\\
\label{eq: noise unbiased part 2}
=
\big(
\mathbb E [ u_k(s,a) \mid \mathcal F'_k ]
- \sigma_k(s,a)
\big)
\big( q_{\pi_k}(s,a)-q_k(s,a) \big)
=
0 .
\end{multline}
Since $g_k(a | s)$ is $\mathcal F'_k$-measurable,
dividing \eqref{eq: noise unbiased part 1} and \eqref{eq: noise unbiased part 2} by $g_k(a | s)$ and summing the results yields $\mathbb E[w_k(s,a) \mid \mathcal F'_k] = 0$.
Applying the tower property then gives
\begin{align}
\label{eq: noise is unbiased}
\mathbb E [ w_k(s,a) \mid \mathcal F_k ]
= \mathbb E \big[ \mathbb E [ w_k(s,a) \mid \mathcal F'_k] \mid \mathcal F_k \big]
= 0 .
\end{align}

Finally,
$g_k^2(a | s) \ge \sigma_{\min}^2$ by the strict exploring starts condition in \eqref{eq: strict exploring starts},
$u_k^2(s,a) \le 1$ 
and $(u_k(s,a) - \sigma_k(s,a))^2 \le 1$ for every state-action pair.
Using $(x+y)^2 \le 2 x^2 + 2 y^2$ on \eqref{eq: noise definition appendix} thus yields
\begin{align*} 
w_k^2(s,a) 
&\le
\frac{1}{\sigma_{\min}^2} \big( 2 v_k^2(s,a) + 2 (q_{\pi_k}(s,a)-q_k(s,a) )^2 \big)
\\
&\le
\frac{1}{\sigma_{\min}^2} \big( 2 v_k^2(s,a) + 4 q_{\pi_k}^2(s,a) + 4 q_k^2(s,a) \big) .
\end{align*}
Since all action-values are bounded,
there exists a constant $c_\polset \in [0,\infty)$ such that $\|q_\pi\|^2 \leq c_\polset$ for every policy $\pi$. 
As a result, using \eqref{eq: properties episode noise} yields
\begin{align}
\nonumber
\mathbb E \!\big[ w_k^2(s,a) \mid \mathcal F_k \big]
&\le
\frac{2}{\sigma_{\min}^2} ( c_v + 2 c_\polset + 2 \|q_k\|^2 ) .
\\
\label{eq: noise has bounded variance}
&\le
\max\left\{ \frac{2}{\sigma_{\min}^2} ( c_v + 2 c_\polset), \frac{4}{\sigma_{\min}^2} \right\} ( 1 + \|q_k\|^2 ) .
\end{align}
Equations \eqref{eq: noise is unbiased} and \eqref{eq: noise has bounded variance} confirm that the perturbations $\{w_k\}_{k \ge 0}$ satisfy the conditions in \eqref{eq: condition on combined noise for proof}.
\end{proof}

\newcommand{\astar}{a_*}
\newcommand{\abar}{\bar a}
\newcommand{\epsES}{\varepsilon_{\rm ES}}
\newcommand{\qhat}{\widehat q}
\newcommand{\calZ}{\mathcal Z}
\newcommand{\calS}{\mathcal S}

\section{13-state First-Visit MCES example}
\label{app:fv}

\subsection{MDP, cycle, and notation}

This appendix describes the discounted 13-state example with 
exploring starts.  The discount and exploring-start perturbation level are
\begin{equation*}
    \gamma=0.9999,
    \qquad
    \epsES=10^{-4}.
\end{equation*}
The "move" action is denoted by $\astar$ and the "stay" action by $\abar$.
A successful "move" sends
\begin{equation*}
    s_i\longmapsto s_{i+1},
\end{equation*}
with indices understood modulo $13$.  See Figure~\ref{fig: long MDP}.
Action $\astar$ gives reward $1$ and succeeds with probability $\tau$; on success it moves to the next state, and on failure the episode terminates.  Action $\abar$ gives reward $0$ and remains in the same state.

A deterministic greedy policy is represented by its "stay" set $S\subseteq\{s_1,\ldots,s_{13}\}$.  Its value function satisfies
\begin{equation*}
V_S(s_i)=
\begin{cases}
\gamma V_S(s_i), & s_i\in S,\\[1mm]
1+\gamma\tau V_S(s_{i+1}), & s_i\notin S,
\end{cases}
\end{equation*}
and the exact policy action values are
\begin{equation*}
q_S(s_i,\astar)=1+\gamma\tau V_S(s_{i+1}),
\qquad
q_S(s_i,\abar)=\gamma V_S(s_i).
\end{equation*}
Because $\gamma<1$, "stay" states satisfy $V_S(s_i)=0$.

The intended adjacent cycle is
\begin{equation}
\label{eq:intended-cycle}
\{s_1\},\{s_1,s_{13}\},\{s_{13}\},\{s_{13},s_{12}\},\ldots,
\{s_2\},\{s_2,s_1\}.
\end{equation}
Equivalently, the local pattern is
\begin{equation*}
    \{s_i\}\longrightarrow \{s_i,s_{i-1}\}\longrightarrow \{s_{i-1}\}.
\end{equation*}
The normal start laws will be chosen later.

There are two requirements for a deterministic cycle. Firstly it must
converge to a point on the boundary and secondly it must follow the
right set of greedy policies as it does so. The first condition we
call the full cycle balance condition and the second the sign-margin
criterion. These are the subjects of the next two sections.

\subsection{General full-cycle balance equations}
\label{sec:balance-general}
We require that, provided the trajectory follows the right policies,
it converges onto the boundary.

Let $\calZ=\{(s_i,a):1\le i\le13,\ a\in\{\astar,\abar\}\}$ be the set
of state-action starts.  For phase $j$ in the intended cycle, let
$S_j$ be the "stay" set, $q_j=q_{S_j}$ its exact action-values, and
$\nu_j$ its nominal start distribution on $\calZ$. We impose a
rotational symmetry, that 
$\nu_{\{s_{i}\}}(s_{i+k},a)$ is a function of $k$ only, so it is only
necessary to specify one distribution over $\calZ$.
With
exploring-starts the  mixture is
\begin{equation*}
    \nu_j^{\epsES}=(1-\epsES)\nu_j+\epsES U_{\calZ},
    \qquad
    U_{\calZ}(z)={1\over 26}.
\end{equation*}
Let $\mu_j^z(s,a)$ be the probability of visiting $(s,a)$ when the episode starts from $z$ and follows policy $S_j$ after the start action.  The phase-$j$ first-visit probability under the perturbed start law is
\begin{equation*}
    \mu_j(s,a)=\sum_{z\in\calZ}\nu_j^{\epsES}(z)\mu_j^z(s,a).
\end{equation*}

For one complete pass through the 26 intended phases, define
\begin{equation*}
    B_a(s)=\sum_{j=1}^{26}\mu_j(s,a),
    \qquad
    A_a(s)=\sum_{j=1}^{26}\mu_j(s,a)q_j(s,a).
\end{equation*}
Here $B_a(s)$ is the expected number of first-visit updates to $(s,a)$ per normal full cycle, and $A_a(s)$ is the corresponding expected return-weighted update mass.

If fractional expected updates are applied over one full cycle, then
\begin{equation*}
    N^+(s,a)=N(s,a)+B_a(s)
\end{equation*}
and
\begin{equation*}
    q^+(s,a)=
    {N(s,a)q(s,a)+A_a(s)\over N(s,a)+B_a(s)}.
\end{equation*}
Equivalently,
\begin{equation*}
    q^+(s,a)-q(s,a)
    ={B_a(s)\over N(s,a)+B_a(s)}
    \left({A_a(s)\over B_a(s)}-q(s,a)\right).
\end{equation*}
Thus the full-cycle target for coordinate $(s,a)$ is $A_a(s)/B_a(s)$.

The boundary cancellation condition asks that, for every state,
\begin{equation}
    {A_{\astar}(s)\over B_{\astar}(s)}=
    {A_{\abar}(s)\over B_{\abar}(s)}=
    \qhat.
\label{eq:abstract-balance}
\end{equation}
i.e. the target is on the boundary. By rotational symmetry it is enough to impose the scalar equation at one state,
\begin{equation}
    F:=A_{\astar}(s_1)B_{\abar}(s_1)-A_{\abar}(s_1)B_{\astar}(s_1)=0,
\label{eq:abstract-F}
\end{equation}
The numerator of $F$ is a high order polynomial in $\tau$ and a low
order polynomial in each component of $\nu$.

The corresponding count ratio is
\begin{equation}
    r={B_{\astar}(s)\over B_{\abar}(s)}.
\label{eq:abstract-ratio}
\end{equation}
Choosing initial counts in this ratio makes the two action coordinates have matching first-order full-cycle learning rates:
\begin{equation*}
    {B_{\astar}(s)\over N_0(s,\astar)}=
    {B_{\abar}(s)\over N_0(s,\abar)}.
\end{equation*}

\subsection{Phase-sign margin LP}
\label{sec:sign-margin-lp}

The balance equations center the deterministic cycle on the boundary, but they do not by themselves ensure that the intended greedy policies occur in the right order.  Let
\begin{equation*}
    \sigma_j(s)=
    \begin{cases}
    +1, & s\in S_j,\\
    -1, & s\notin S_j.
    \end{cases}
\end{equation*}
At the boundary level $\qhat$, define the scaled expected gap increment in phase $j$ by
\begin{equation}
    g_j(s)=
    \mu_j(s,\abar)\bigl(q_j(s,\abar)-\qhat\bigr)
    -{\mu_j(s,\astar)\over r}\bigl(q_j(s,\astar)-\qhat\bigr).
\label{eq:gap-increment}
\end{equation}
Let
\begin{equation*}
    C_k(s)=\sum_{j<k}g_j(s)
\end{equation*}
be the cumulative gap displacement before phase $k$.  The scaled
initial gap $d_0(s)$ is then chosen by
\begin{equation}
\begin{aligned}
\text{maximize}\quad & \rho,\\
\text{over}\quad & d_0\in\mathbb R^{13},\ \rho\in\mathbb R,\\
\text{subject to}\quad
&\sigma_k(s)\bigl(d_0(s)+C_k(s)\bigr)\ge \rho,
\qquad k=1,\ldots,26,
\quad s=s_1,\ldots,s_{13}.
\end{aligned}
\label{eq:sign-lp}
\end{equation}
A positive optimal value $\rho$ means that the phases are followed in
the intended order.

The initial points are then placed symmetrically around the boundary:
\begin{equation*}
    q_0(s,\astar)=\qhat-{d_0(s)\over 2N_0(s,\abar)},
    \qquad
    q_0(s,\abar)=\qhat+{d_0(s)\over 2N_0(s,\abar)}.
  \end{equation*}
The factor $1/2$ splits the desired total action gap equally above and below $\widehat q$.

\subsection{Rotational start search and selected start law}
\label{sec:start-search}
The initial search for a suitable start distribution $\nu$ was over a
wide class of rotationally symmetric distributions. Distributions with
small support were preferred, as these would avoid adding too much
extra noise into the full simulations. We started by
considering all $13^2$ pairs of deterministic starts. For each, we solved
\eqref{eq:abstract-F} for $\tau$, and for any roots in $(0,1]$ we
then solved  \eqref{eq:sign-lp} for $\rho$. There were no solutions
with a positive $\rho$.

We then broadened the search to distributions with a two point support
in the two-"stay" phases. That is, 
for offsets $o_1,o_2,o_3$ we set
\begin{align}
    \nu_{\{s_i\}}^{(o_1)}
    &=\delta_{(s_{i+o_1},\abar)},\\
    \nu_{\{s_i,s_{i-1}\}}^{(o_2,o_3,\lambda)}
    &=\lambda\delta_{(s_{i+o_2},\abar)}+(1-
    \lambda)\delta_{(s_{i+o_3},\astar)}.
\end{align}
For a fixed triple and trial $\tau$, equation~\eqref{eq:abstract-F} is
solved for a root $\lambda\in[0,1]$ and then the sign-margin
LP~\eqref{eq:sign-lp} is evaluated. The $\tau$ which gives the maximum margin $\rho$ is then
found by gridding. The offsets
$    (o_1,o_2,o_3)=(-6,+1,-1)$ were chosen as giving a good margin
with the corresponding $\lambda$ close to 1.

Thus
\begin{align}
\nu_{\{s_i\}}&=\delta_{(s_{i-6},\bar a)},\\
\nu_{\{s_i,s_{i-1}\}}&=\lambda\delta_{(s_{i+1},\bar a)}+(1-
\lambda)\delta_{(s_{i-1},a_*)}.
\end{align}
and the corresponding $\tau$ and $\lambda$ are given by
\begin{equation*}
    \tau=0.512333524979486,
    \qquad
    \lambda=0.923714016854384.
\end{equation*}
The resulting boundary quantities are
\begin{equation*}
    \qhat=1.971056244749326
    \qquad
    q_*={1\over 1-
    \gamma\tau}=2.050366396036596
    \qquad
    r=2.010039006519156.
\end{equation*}
The sign-margin LP gives
\begin{equation*}
    \rho=8.527647265071263\times10^{-3}
  \end{equation*}
There were no  distributions with a two point support
in the single-"stay" phases and a one point support in the two-"stay" phases which admitted a positive margin.

Note that this procedure is not specific to the 13 state MDP, it in fact gives a suitable three point start distribution for the corresponding cycle on examples with 5 states and above, although a larger number of states gives a larger margin.

Figure~\ref{fig:spiral} in the main text shows the resulting
deterministic cycle. Note that the corners of the cycle contract
linearly towards $(\widehat q,\widehat q)$, and so stay in the right greedy region, thus the cycles will persist indefinitely. The next section confirms this contraction.

\subsection{Full-cycle linear interpolation toward $\widehat q$}
\label{sec:linear-interpolation}

Note that the full-cycle balance equations also imply a geometric contraction.  Suppose counts are exactly proportional to the full-cycle masses,
\begin{equation*}
    N_0(s,\astar)=c_sB_{\astar}(s),
    \qquad
    N_0(s,\abar)=c_sB_{\abar}(s).
\end{equation*}
Then after $m$ complete normal cycles,
\begin{equation*}
    N_{26m}(s,a)=(c_s+m)B_a(s).
\end{equation*}
Using $A_a(s)=\qhat B_a(s)$, the next full-cycle endpoint is
\begin{align}
q_{26(m+1)}(s,a)
&={ (c_s+m)B_a(s)q_{26m}(s,a)+A_a(s) \over (c_s+m+1)B_a(s)}\\
&={c_s+m\over c_s+m+1}q_{26m}(s,a)+{1\over c_s+m+1}\qhat.
\end{align}
Therefore the two-action point
\begin{equation*}
    P_m(s)=\bigl(q_{26m}(s,\astar),q_{26m}(s,\abar)\bigr)
\end{equation*}
satisfies
\begin{equation}
    P_{m+1}(s)=(1-\alpha_m)P_m(s)+\alpha_m(\qhat,\qhat),
    \qquad
    \alpha_m={1\over c_s+m+1}.
\label{eq:linear-interpolation}
\end{equation}
Equivalently,
\begin{equation*}
    P_m(s)-(\qhat,\qhat)=
    {c_s\over c_s+m}\bigl(P_0(s)-(\qhat,\qhat)\bigr).
\end{equation*}
Thus the full-cycle endpoints move on the line segment toward the boundary point. The intermediate phase points need not lie precisely on this line, but deviate only by order $1/k$.
\subsection{First-visit probabilities}

Fix an anchor state $s_i$ and write $d=(m-i)\bmod 13$, so $s_m=s_{i+d}$.  Let $\mu_1^\varepsilon$ and $\mu_2^\varepsilon$ denote the perturbed first-visit probabilities in the one-"stay" and two-"stay" phases for the selected start law. The values for the current $\tau$ and $\lambda$ are given in the following table for convenience. 
\begin{table}[H]
\centering
\small
\caption{Single-phase first-visit probabilities by relative distance $d$ for the tuned $n=13$ start law.}
\vspace{0.3cm}
\label{tab:first-visit-probabilities}
\begin{filecontents*}{figures/first-visit-probabilities.csv}
d,mu1_move,mu1_stay,mu2_move,mu2_stay
0,3.846153846153847e-06,1.809503953663139e-02,3.846153846153847e-06,3.908577493225771e-02
1,3.531135969193039e-02,3.846153846153847e-06,7.628220070114757e-02,6.015469413686563e-04
2,6.890758785627323e-02,3.846153846153847e-06,1.166624392862900e-03,3.846153846153847e-06
3,1.344825044414958e-01,3.846153846153847e-06,2.262065683125063e-03,3.846153846153847e-06
4,2.624751369514378e-01,3.846153846153847e-06,4.400206633994957e-03,3.846153846153847e-06
5,5.122980087126930e-01,3.846153846153847e-06,8.573544599640492e-03,3.846153846153847e-06
6,9.999156631913028e-01,9.999038461538462e-01,1.671928904572693e-02,3.846153846153847e-06
7,1.555799732373936e-05,3.846153846153847e-06,3.261858910892813e-02,3.846153846153847e-06
8,1.535267408422398e-05,3.846153846153847e-06,6.365169408451568e-02,3.846153846153847e-06
9,1.495191319409171e-05,3.846153846153847e-06,1.242237696222821e-01,3.846153846153847e-06
10,1.416968663542893e-05,3.846153846153847e-06,2.424515891665755e-01,3.846153846153847e-06
11,1.264289496452645e-05,3.846153846153847e-06,4.732149762570987e-01,3.846153846153847e-06
12,9.662821249921100e-06,3.846153846153847e-06,9.236313082739486e-01,9.236254916065448e-01
\end{filecontents*}
\pgfplotstabletypeset[
  col sep=comma,
  columns/d/.style={int detect,column name={$d$}},
  columns/{mu1_move}/.style={sci,sci zerofill,precision=6,column name={$\mu_1^\varepsilon(d,\textup{"move"})$}},
  columns/{mu1_stay}/.style={sci,sci zerofill,precision=6,column name={$\mu_1^\varepsilon(d,\textup{"stay"})$}},
  columns/{mu2_move}/.style={sci,sci zerofill,precision=6,column name={$\mu_2^\varepsilon(d,\textup{"move"})$}},
  columns/{mu2_stay}/.style={sci,sci zerofill,precision=6,column name={$\mu_2^\varepsilon(d,\textup{"stay"})$}},
  every head row/.style={before row=\toprule,after row=\midrule},
  every last row/.style={after row=\bottomrule}
]{figures/first-visit-probabilities.csv}
\end{table}

\subsection{Sampled recovery strategy}
\label{sec:recovery-strategy}

The balance equations and margin LP describe the intended normal
cycle. In the sampled first-visit MC run, however, return noise,
trajectory noise, and rare exploring starts can move the greedy policy off the intended 26-phase cycle. The simulation therefore has to choose start distributions even when the observed greedy "stay"-set is not one of the intended cycle phases.

Let
\begin{equation*}
    S=\{s:q(s,\bar a)>q(s,a_*)\}
\end{equation*}
denote the set of states for which the "stay" action is greedy at the start of an update. If $S$ is one of the intended cycle phases, the controller uses the normal start law from Section~\ref{sec:start-search}.
If $S$ is off-cycle, the rollout policy is the current greedy policy with "stay" set $S$, and the controller chooses from the $26$ state-action starts as follows:

Let the current action gap be
\begin{equation*}
    G(s)=q(s,\bar a)-q(s,a_*).
\end{equation*}
If $S$ is empty then all $G(s)$ are negative, in this case we choose $s_0=\arg\max_s G(s)$ as the state whose action values are nearest the "move"/"stay" boundary. For other non-cycle phases then for $s_i\in S$ we let $m(s_i)= {\mathbf 1}(s_{i-1}\in M)+ {\mathbf 1}(s_{i+1}\in M)$ denote the number of neighbours of $s_i$ in $S$ and the $S_{\text{min}}=\{s\in S:m(s)=\min m\}$. That is, $S_{\text{min}}$ is the subset of $S$ consisting of the states with the minimum number of neighbours (all operations are modulo 13 and so $s_{1}$ and $s_{13}$ are considered neighbours). Finally we let $s_0=\arg\min_{s\in S_{\text{min}}} G(s)$
be the "stay" state with the fewest number of "stay" neighbours whose action values are closest to the "move"/"stay" boundary. This guarantees, for example, that if $M$ consists of an adjacent pair and one or more isolated states then one of the isolated states is chosen, and in a run of adjacent states then the two ends are chosen in preference to an internal state. Finally, the starting pair is chosen as the "stay" action in this state $(s_0,\bar a)$, in an attempt to move the action value across the boundary. If successful this would create a "stay" state if $S$ is empty or reduce a "stay" state, moving $S$ closer to to a cycle phase, otherwise. If at the next step $S$ is still not a cycle phase then the process is repeated.      

For the simulations we present, approximately $83.6\%$ of updates are
normal on-cycle updates, $14.9\%$ are attempted recoveries from the
optimal all-"move" policy and the remaining $1.5\%$ are recoveries from
all other non-cycle policies.  

\textit{ChapGPT 5.5 Pro was used to help refine and help code this example. The core simulation was hand-coded.}



\end{document}